\pdfoutput=1
\documentclass[11pt]{article}

\usepackage[preprint]{acl}

\usepackage{times}
\usepackage{latexsym}

\usepackage{amsmath}
\usepackage{amssymb}
\usepackage{amsthm}
\usepackage{booktabs}
\usepackage{multirow}
\usepackage{algorithm}
\usepackage{algorithmicx}
\usepackage{pgfplots}
\usepackage{tikz}
\usepackage{algpseudocode}
\usepackage{algcompatible}
\usepackage{algpseudocode}
\newtheorem{theorem}{Theorem}

\usepackage{bera}% optional: just to have a nice mono-spaced font
\usepackage{listings}
\usepackage{xcolor}

\definecolor{eclipseStrings}{RGB}{42,0.0,255}
\definecolor{eclipseKeywords}{RGB}{127,0,85}
\colorlet{numb}{magenta!60!black}

\lstdefinelanguage{json}{
    basicstyle=\normalfont\ttfamily,
    commentstyle=\color{eclipseStrings}, % style of comment
    stringstyle=\color{eclipseKeywords}, % style of strings
    numbers=left,
    numberstyle=\scriptsize,
    stepnumber=1,
    numbersep=8pt,
    showstringspaces=false,
    breaklines=true,
    frame=lines,
    backgroundcolor=\color{white}, %only if you like
    string=[s]{"}{"},
    comment=[l]{:\ "},
    morecomment=[l]{:"},
    literate=
        *{0}{{{\color{numb}0}}}{1}
         {1}{{{\color{numb}1}}}{1}
         {2}{{{\color{numb}2}}}{1}
         {3}{{{\color{numb}3}}}{1}
         {4}{{{\color{numb}4}}}{1}
         {5}{{{\color{numb}5}}}{1}
         {6}{{{\color{numb}6}}}{1}
         {7}{{{\color{numb}7}}}{1}
         {8}{{{\color{numb}8}}}{1}
         {9}{{{\color{numb}9}}}{1}
}

\usepackage[T1]{fontenc}
\usepackage[utf8]{inputenc}

\usepackage{microtype}

\usepackage{inconsolata}

\usepackage{graphicx}

\title{Everything Can Be Described in Words: A Simple Unified Multi-Modal Framework with Semantic and Temporal Alignment}

\author{Xiaowei Bi \\
  Northwestern University \\
  \texttt{xiaoweibi2021@u.northwestern.edu} \\\And
  Zheyuan Xu \\
  IEEE Member \\
  \texttt{cx1014@uw.edu} \\}

\begin{document}
\maketitle
\begin{abstract}
While multi-modal learning has advanced significantly, current approaches often create inconsistencies in representation and reasoning of different modalities. We propose UMaT, a theoretically-grounded framework that unifies visual and auditory inputs as structured text for large language models, addressing semantic alignment, temporal synchronization, and efficient sparse information retrieval. It significantly improves state-of-the-art Long Video Question Answering accuracy (up to 13.7\%, and 16.9\% on long videos) via redundancy minimization and structured textual representation for unified multi-modal reasoning.
% We introduce UMaT (Unified Multi-modal as Text), a theoretically-grounded framework that unifies visual and auditory inputs into a structured textual space for seamless processing with large language models. UMaT addresses three key challenges: (1) semantic alignment across modalities, (2) temporal synchronization, and (3) efficient retrieval of sparse information from long sequences. We prove its optimality for context selection under information-theoretic constraints. Experimental results on Long Video Question Answering show that UMaT improves state-of-the-art models by up to 13.7\% in overall accuracy, with particularly strong gains (16.9\%) on long-duration videos. Our framework introduces a novel method for redundancy minimization while preserving rare signals and demonstrate how structured textual representation can serve as a unifying abstraction for multi-modal reasoning.
\end{abstract}

\section{Introduction}

Integrating heterogeneous modalities (vision, audio, text) remains a key AI challenge, despite progress in large-scale multi-modal models \cite{Guo2019MultimodalRL,Wang2022InternVideoGV,Ye2023mPLUGOwlME}. However, separate encoding pathways often lead to representational issues, scalability limits, and poor transfer \cite{Fu2021VIOLETE,Yang2022}.

Long Video Question Answering (LVQA) highlights these challenges due to scattered information, modality alignment needs, and computational constraints. Existing LVQA methods truncate content or use resource-intensive architectures prone to semantic drift and redundancy, largely due to a lack of unified cross-modal representation.

We introduce UMaT (Unified Multi-modal as Text) with the following contributions:
\begin{itemize}
    \item Formulation of multi-modal alignment via textual representation with optimal information preservation and minimized redundancy.
    \item Novel temporal segmentation and contextual alignment for semantic coherence and efficient retrieval-augmented generation.
    \item Information-theoretic content deduplication adaptively filtering redundancy while preserving critical signals with provable retention.
    \item Significant improvement in state-of-the-art LVQA performance, particularly on long, information-sparse videos.
\end{itemize}

UMaT offers a modular and interpretable solution adaptable to various scales and domains beyond LVQA, unlike black-box methods.

\section{Related Work}

% \subsection{Retrieval-Augmented Generation for Multi-Modal Tasks}
% Recent work increasingly uses retrieval-augmented generation (RAG) to improve multi-modal large language models. Wang et al. \cite{wang2023filling} enabled proactive image questioning for better open-domain visual QA. Lin and Byrne \cite{lin-byrne-2022-retrieval} integrated differentiable dense passage retrieval with generation, outperforming separate retrievers. Lin et al. \cite{lin-etal-2023-fvqa} extended this to adversarial fact-based visual QA. PreFLMR \cite{lin-etal-2024-preflmr} achieved state-of-the-art knowledge-based visual QA via pre-trained fine-grained late-interaction retrieval.

% For video understanding specifically, Video-RAG \cite{luo2024videoragvisuallyalignedretrievalaugmentedlong} uses visually-aligned auxiliary texts to improve large visual-language models. While these retrieval-based methods are valuable for multi-modal tasks, they often process modalities separately, lacking a unified representation.
Retrieval-Augmented Generation (RAG) is increasingly used to enhance multi-modal large language models, as demonstrated by Wang et al.'s \cite{wang2023filling} proactive image questioning for visual QA, Lin and Byrne's \cite{lin-byrne-2022-retrieval} integrated differentiable dense retrieval, Lin et al.'s \cite{lin-etal-2023-fvqa} work on adversarial visual QA, and PreFLMR's \cite{lin-etal-2024-preflmr} state-of-the-art knowledge-based visual QA. Video-RAG \cite{luo2024videoragvisuallyalignedretrievalaugmentedlong} applies this to video using visually-aligned text to improve visual-language models; however, these retrieval methods often treat modalities as separate streams without a unified representation.

% \subsection{Unified Representation of Multi-Modal Data}

Creating unified representations across modalities with differing dimensionality and semantics is challenging. Cross-Modal Generalization \cite{xia2024achieving} and training-free methods \cite{Huang2024UnlockingTP} address this, while Zhu and Li \cite{Zhu2023IterativeUA} use contrastive learning for image-text retrieval, and Video-XL \cite{Shu2024VideoXLEV} condenses video using Visual Summarization Tokens. However, these methods often involve complex architectures and lack theoretical grounding for optimal cross-modal information preservation.

% \subsection{Temporal Segmentation and Content Deduplication}

% Efficient processing of temporal data requires effective segmentation and deduplication strategies. Tirumala et al. \cite{TirumalaD4} demonstrated that careful data selection via pre-trained model embeddings can significantly improve LLM training efficiency and downstream performance. Liu et al. \cite{liu2023one} introduced dynamic video token masking and masked video modeling to reduce sequence lengths while improving robustness for videos of varying duration.

% For finer temporal understanding, Momentor \cite{Momenter2024} enables segment-level reasoning and localization via automatic data generation. SlowFast-LLaVA \cite{Xu2024SlowFastLLaVAAS} uses a two-stream design for detailed spatial and long-range temporal context in a training-free video LLM. A Hierarchical Keyframe Selector \cite{Park2024TooMF} progressively reduces keyframes for lower cost, and an iterative approach \cite{VideoAgent2024} gathers relevant information. While offering temporal processing insights, these methods lack a unified theoretical framework for optimal information preservation and efficiency.
Efficient temporal data processing benefits from segmentation and deduplication, as shown by Tirumala et al. \cite{TirumalaD4} using pre-trained embeddings for LLM training and Liu et al.'s \cite{liu2023one} dynamic token masking for variable-length video robustness. For finer understanding, Momentor \cite{Momenter2024} enables segment-level reasoning, SlowFast-LLaVA \cite{Xu2024SlowFastLLaVAAS} uses a two-stream design for spatial and temporal context, a Hierarchical Keyframe Selector \cite{Park2024TooMF} reduces keyframes, and an iterative approach \cite{VideoAgent2024} gathers relevant information; however, these temporal processing methods lack a unified framework for optimal information and efficiency balance.

% \subsection{Multi-Modal Content Integration and Long-Form Video Understanding}
% Long-form video understanding presents unique challenges due to extended temporal dependencies and information sparsity. Weng et al. \cite{LongVLM2024} introduced LongVLM, which decomposes long videos into short-term segments and encodes local features via hierarchical token merging. Zhang et al. \cite{llovi} proposed LLoVi, which uses short-term visual captioning to generate textual descriptions of video clips.

% Other approaches focus on specific aspects of video understanding, such as IntentQA \cite{IntentQA} for intent reasoning, PLLaVA \cite{xu2024pllava} for addressing vulnerability to prompt changes, MovieChat+ \cite{moviechatplus} for question-aware sparse memory, and TimeChat \cite{timechat2023} which incorporates time-aware frame encoding to enhance temporal localization. These approaches, however, typically lack interpretability and cross-modal unified representation.
Long-form video understanding faces challenges due to long temporal dependencies and sparse information \cite{temp_1}, \cite{temp_2}, \cite{temp_3}. LongVLM \cite{LongVLM2024} addresses this by segmenting videos and using hierarchical token merging for local features. LLoVi \cite{llovi} generates textual descriptions of short video clips via captioning. Other methods focus on specific tasks like intent reasoning (IntentQA \cite{IntentQA}), prompt robustness (PLLaVA \cite{xu2024pllava}), question-aware memory (MovieChat+ \cite{moviechatplus}), and temporal localization (TimeChat \cite{timechat2023}). However, these approaches often lack interpretability and a unified cross-modal representation.

\section{Problem Formulation}

\subsection{Multi-Modal Alignment Problem}
% We begin by formalizing the multi-modal alignment problem. 
Given a video $V$ with visual content $V_v$ and audio content $V_a$, we aim to create a unified representation that preserves the semantics of both modalities (visual and audio) $\mathcal{M} = \{v, a\}$. For each modality $m \in \mathcal{M}$, we have a sequence of raw inputs $X^m = \{x^m_1, x^m_2, ..., x^m_{T_m}\}$ where $T_m$ is the length of the sequence for modality $m$. For visual content, $x^v_i$ represents a frame or short clip, while for audio, $x^a_i$ represents a short audio segment.

The multi-modal alignment problem can be formulated as finding mapping functions $f_m: X^m \rightarrow \mathcal{T}$ for each modality $m$ that transform inputs into a shared textual space $\mathcal{T}$, such that:

\begin{equation}
\begin{aligned}
\min_{f_v, f_a} \quad & \mathcal{L}(f_v, f_a) \\
\text{s.t.} \quad & I(X^m; f_m(X^m)) \geq (1-\epsilon)H(X^m),\quad \\
\forall m \in \mathcal{M}
\end{aligned}
\end{equation}

where $I(\cdot;\cdot)$ represents mutual information, $H(\cdot)$ represents entropy, $\epsilon$ is a small constant determining allowable information loss, and $\mathcal{L}$ is a loss function measuring semantic alignment between modalities. The formulation ensures that the textual representations preserve as much information as possible from the original modalities while staying within the textual space.

\subsection{Temporal Alignment and Segmentation}

% Given the textual representations of each modality, we must align them temporally and segment them into manageable chunks for efficient processing. Let $\tau^m = \{t^m_1, t^m_2, ..., t^m_{T_m}\}$ represent the timestamps associated with each element in the sequence $X^m$.

The temporal alignment problem involves creating aligned segments $S = \{S_1, S_2, ..., S_K\}$ where each segment $S_k$ contains temporally aligned textual representations from all modalities:

\begin{equation}
S_k = \{(f_m(x^m_i), t^m_i) \mid t^m_i \in [t_k^{\text{start}}, t_k^{\text{end}}], m \in \mathcal{M}\}
\end{equation}

where $[t_k^{\text{start}}, t_k^{\text{end}}]$ defines the time interval for segment $k$.

\subsection{Information-Theoretic Content Selection}

% For long videos, the number of segments may exceed the processing capacity of language models. Therefore, we need an efficient content selection mechanism that retrieves the most relevant segments for a given query $q$.

% We formulate this as an information-theoretic optimization problem. Let $\mathcal{P}(S)$ represent the power set of all segments. The content selection problem involves finding a subset of segments $S_q \subset S$ that maximizes the conditional mutual information with the query response $r$ given the query $q$:
Processing long videos with numerous segments can exceed language model capacity, necessitating efficient relevant segment retrieval for a given query $q$. We formulate this as an information-theoretic optimization problem to find a segment subset $S_q \subset S$ that maximizes the conditional mutual information with the query response $r$ given $q$, with $\mathcal{P}(S)$ being the segment power set:

\begin{equation}
\begin{aligned}
S_q^* = \arg\max_{S_q \in \mathcal{P}(S)} \quad & I(S_q; r \mid q) \\
\text{s.t.} \quad & |S_q| \leq L
\end{aligned}
\end{equation}

where $L$ is the maximum number of segments that can be processed within computational constraints. This formulation aims to select segments with most information about the desired response given the query, while respecting computational constraints.

\section{Method}

\subsection{Unified Multi-Modal Textual Representation}

\begin{figure*}[!htbp]
    \centering
    \includegraphics[width=\linewidth]{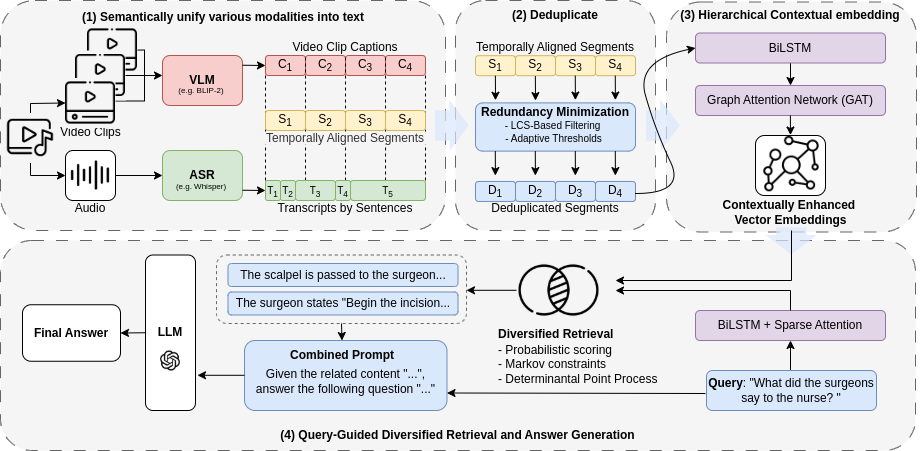}
    % \caption{Flowchart showing the UMaT framework. Raw videos are first processed through parallel modality-specific pathways: a pre-trained vision-language model (VLM) for visual content and an automatic speech recognition (ASR) model for audio. These textual representations are temporally aligned and fused into coherent segments. Our hierarchical contextual embedding engine transforms these segments into a high-dimensional vector space while preserving temporal and semantic relationships. During inference, our probabilistic diversified retrieval mechanism selects the most relevant segments based on the query, which are then assembled into a prompt for the large language model.}
    \caption{Flowchart showing UMaT framework: Raw video undergoes parallel processing via a VLM (visual) and ASR (audio). Resulting text is temporally aligned and fused into segments. A hierarchical contextual embedding engine creates high-dimensional vectors preserving temporal and semantic relationships. During inference, probabilistic diversified retrieval selects relevant segments for an LLM prompt.}
    \label{fig:architecture}
\end{figure*}

% UMaT transforms visual and auditory inputs into a unified textual space through specialized encoding functions. For visual content, we employ a vision-language model $\phi_v$ that maps video frames or clips to detailed textual descriptions. For audio content, we utilize an automatic speech recognition (ASR) model $\phi_a$ that transcribes spoken language into text.

% Formally, given a long video $V$, we first segment it into $N$ non-overlapping short clips $\{v_1, v_2, ..., v_N\}$, where each clip $v_i \in \mathbb{R}^{T_v \times H \times W \times 3}$ contains $T_v$ frames with height $H$ and width $W$. For each clip $v_i$, we generate a textual caption $c_i = \phi_v(v_i)$.

% Simultaneously, we process the corresponding audio segments $\{a_1, a_2, ..., a_N\}$ using the ASR model to produce transcriptions $t_i = \phi_a(a_i)$. Each transcription is structured as $(t_i^{\text{start}}, t_i^{\text{end}}, t_i^{\text{text}})$, capturing the start time, end time, and transcribed content.

% To ensure rich and structured visual descriptions, we employ a carefully designed prompting strategy $\mathcal{P}$ that emphasizes objective scene content over subjective interpretations:

% \begin{equation}
% c_i = \phi_v(v_i, \mathcal{P})
% \end{equation}

% where $\mathcal{P}$ instructs the model to focus on scene composition, object attributes, actions, and spatial relationships.
UMaT unifies visual and auditory inputs into text using encoders $\phi_v$ and $\phi_a$.

For a video $V$ segmented into clips $v_i \in \mathbb{R}^{T_v \times H \times W \times 3}$, textual captions $c_i$ are generated by $\phi_v(v_i)$ with a prompting strategy $\mathcal{P}$, which emphasizes objective scene content (composition, attributes, actions, spatial relations):
\begin{equation*}
c_i = \phi_v(v_i, \mathcal{P})
\end{equation*}

Corresponding audio segments $a_i$ are transcribed into $t_i = \phi_a(a_i)$, structured as $(t_i^{\text{start}}, t_i^{\text{end}}, t_i^{\text{text}})$.

\subsection{Information-Preserving Temporal Alignment}

% To create a coherent representation of the video content, we need to align the visual captions and audio transcriptions temporally. We formulate this as an optimization problem that maximizes temporal coherence while preserving semantic content.

% Given the sets of visual captions $\{c_1, c_2, ..., c_N\}$ and audio transcriptions $\{t_1, t_2, ..., t_N\}$ with associated timestamps, we aggregate them into fixed-length temporal segments $\{S_1, S_2, ..., S_K\}$ where:

% \begin{equation}
% S_k = \bigcup_{i=kT_s}^{(k+1)T_s-1} \{(c_i, t_i)\}
% \end{equation}

% where $T_s$ is the segment size parameter.

% To optimize information preservation within each segment, we introduce an information density measure $\rho(S_k)$ that quantifies the amount of non-redundant information contained in segment $S_k$:

% \begin{equation}
% \rho(S_k) = \frac{H(S_k)}{\sum_{i=kT_s}^{(k+1)T_s-1} |c_i| + |t_i|}
% \end{equation}

% where $H(S_k)$ represents the entropy of the segment and $|c_i|, |t_i|$ represent the length of the caption and transcription, respectively.

% We then optimize the segment size parameter $T_s$ to maximize the minimum information density across all segments:

% \begin{equation}
% T_s^* = \arg\max_{T_s} \min_{k \in \{1,2,...,K\}} \rho(S_k)
% \end{equation}

% This formulation ensures that segments maintain a balance between length and information content, avoiding both overly sparse and overly dense representations.

To align visual captions $\{c_i\}$ and audio transcriptions $\{t_i\}$ temporally for coherent video representation, we segment them into fixed-length temporal segments $S_k$:
\begin{equation*}
S_k = \bigcup_{i=kT_s}^{(k+1)T_s-1} \{(c_i, t_i)\}
\end{equation*}
where $T_s$ is the segment size.

Information density $\rho(S_k)$ within each segment is defined as:
\begin{equation*}
\rho(S_k) = \frac{H(S_k)}{\sum_{i=kT_s}^{(k+1)T_s-1} |c_i| + |t_i|}
\end{equation*}
where $H(S_k)$ is the entropy and $|c_i|, |t_i|$ are lengths.

The optimal segment size $T_s^*$ maximizes the minimum information density across all segments:
\begin{equation*}
T_s^* = \arg\max_{T_s} \min_{k} \rho(S_k)
\end{equation*}
This balances segment length and information content.

\subsection{Optimal Redundancy Minimization}

% Within each segment, we need to minimize redundancy while preserving critical information. We formulate this as a constrained optimization problem.

% For consecutive captions $c_i$ and $c_{i+1}$ within a segment, we define a similarity measure $\text{Sim}(c_i, c_{i+1})$ based on the longest common subsequence (LCS):

% \begin{equation}
% \text{Sim}(c_i, c_{i+1}) = \frac{2 \cdot |\text{LCS}(c_i, c_{i+1})|}{|c_i| + |c_{i+1}|}
% \end{equation}

% We then identify and remove redundant content while preserving unique information through an adaptive thresholding mechanism. Let $c_i'$ represent the processed caption after redundancy removal. We compute $c_i'$ as:

% \begin{equation}
% c_i' = \begin{cases}
% c_i - \text{LCS}(c_{i-1}, c_i) & \text{if } \text{Sim}(c_{i-1}, c_i) > \theta(c_i) \\
% c_i & \text{otherwise}
% \end{cases}
% \end{equation}

% where $\theta(c_i)$ is an adaptive threshold function that depends on the information content of the caption:

% \begin{equation}
% \theta(c_i) = \alpha + \beta \cdot \frac{H(c_i)}{|c_i|}
% \end{equation}

% with $\alpha$ and $\beta$ being hyperparameters controlling the base threshold and the influence of information density, respectively.
Within each segment, redundancy is minimized while preserving critical information via constrained optimization.

Similarity between consecutive captions $c_i$ and $c_{i+1}$ is measured by:
\begin{equation*}
\text{Sim}(c_i, c_{i+1}) = \frac{2 \cdot |\text{LCS}(c_i, c_{i+1})|}{|c_i| + |c_{i+1}|}
\end{equation*}

In which LCS denotes longest-common subsequences. Redundant content is removed based on an adaptive threshold $\theta(c_i)$:
\begin{equation*}
c_i' = \begin{cases}
c_i - \text{LCS}(c_{i-1}, c_i) & \text{if } \text{Sim}(c_{i-1}, c_i) > \theta(c_i) \\
c_i & \text{otherwise}
\end{cases}
\end{equation*}
where $\theta(c_i)$ adapts to caption information content:
\begin{equation*}
\theta(c_i) = \alpha + \beta \cdot \frac{H(c_i)}{|c_i|}
\end{equation*}
$\alpha$ and $\beta$ are hyperparameters.

\begin{theorem}
% Given captions $\{c_1, c_2, ..., c_n\}$ and their processed versions $\{c_1', c_2', ..., c_n'\}$ according to our redundancy minimization method, if the adaptive threshold function $\theta(c_i)$ satisfies $\theta(c_i) < 1 - \frac{H_{\text{unique}}(c_i)}{H(c_i)}$ for all $i$, then the processed captions preserve at least $(1-\epsilon)$ of the unique information content from the original captions.
Given captions $\{c_i\}$ and processed $\{c_i'\}$ with adaptive threshold $\theta(c_i) < 1 - \frac{H_{\text{unique}}(c_i)}{H(c_i)}$, the processed captions preserve at least $(1-\epsilon)$ of the unique information.
\end{theorem}

\begin{proof}
% (Sketch) The unique information content in caption $c_i$ can be quantified as $H_{\text{unique}}(c_i) = H(c_i) - I(c_i; c_{i-1})$. When we remove redundant content based on the similarity threshold, we eliminate information captured by $I(c_i; c_{i-1})$. By setting the threshold appropriately, we ensure that the preserved information is at least $(1-\epsilon)$ of the unique information.
Unique information $H_{\text{unique}}(c_i) = H(c_i) - I(c_i; c_{i-1})$. Redundancy removal eliminates $I(c_i; c_{i-1})$. The threshold ensures preserved information is at least $(1-\epsilon) H_{\text{unique}}(c_i)$.
\end{proof}

This guarantees critical information preservation during efficient content reduction.

\subsection{Retrieval-Augmented Generation with Dense Embeddings}

% To enable efficient retrieval of relevant video segments for a given query, we formulate a context-aware hierarchical embedding approach combined with a probabilistic selection mechanism. Our method addresses three key challenges: (1) semantic alignment between query and content, (2) contextual coherence across retrieved segments, and (3) diversity in information coverage.
For efficient query-based video segment retrieval, we use a context-aware hierarchical embedding with probabilistic selection, addressing semantic alignment, contextual coherence, and information diversity.

\subsubsection{Hierarchical Contextual Embedding}

Let $\mathcal{S} = \{S_1', S_2', \ldots, S_K'\}$ be the set of all processed segments. Each segment $S_k'$ consists of a sequence of textual elements $S_k' = \{w_1^k, w_2^k, \ldots, w_{n_k}^k\}$ representing words or tokens from both visual captions and audio transcriptions. We first employ a hierarchical embedding approach that captures both local and global semantic structure:

\begin{equation}
\mathbf{h}_i^k = \text{BiLSTM}(w_i^k, \mathbf{h}_{i-1}^k), \quad i \in \{1, 2, \ldots, n_k\}
\end{equation}

\begin{equation}
\mathbf{a}_i^k = \frac{\exp(\mathbf{v}^T \tanh(\mathbf{W}\mathbf{h}_i^k + \mathbf{b}))}{\sum_{j=1}^{n_k} \exp(\mathbf{v}^T \tanh(\mathbf{W}\mathbf{h}_j^k + \mathbf{b}))}
\end{equation}

\begin{equation}
\mathbf{e}_k^{\text{local}} = \sum_{i=1}^{n_k} \mathbf{a}_i^k \mathbf{h}_i^k
\end{equation}

where $\mathbf{h}_i^k$ is the hidden state of the BiLSTM for token $i$ in segment $k$, $\mathbf{a}_i^k$ is the attention weight, and $\mathbf{e}_k^{\text{local}}$ is the local segment embedding.

To capture global contextual information and temporal relationships, we introduce a graph-based contextual enhancement:

\begin{equation}
\mathcal{G} = (\mathcal{V}, \mathcal{E}, \mathbf{A})
\end{equation}

where $\mathcal{V} = \{v_1, v_2, \ldots, v_K\}$ is the set of vertices corresponding to segments, $\mathcal{E}$ is the set of edges representing temporal proximity and semantic similarity, and $\mathbf{A} \in \mathbb{R}^{K \times K}$ is the adjacency matrix with elements:
% \begin{equation}
% \footnotesize
% A_{i,j} =
% \begin{cases}
%   \begin{aligned}
%     &\alpha \cdot \exp(-\beta |i-j|) \\
%     &+ (1 - \alpha) \cdot \cos(\mathbf{e}_i^{\text{local}}, \mathbf{e}_j^{\text{local}})
%   \end{aligned}
%   & \text{if } |i-j| \leq \delta \text{ or } \cos(\mathbf{e}_i^{\text{local}}, \mathbf{e}_j^{\text{local}}) \geq \tau \\
%   0 & \text{otherwise}
% \end{cases}
% \end{equation}
{
\footnotesize
\begin{equation}
\begin{aligned}
A_{i,j} =
&\alpha \cdot \exp(-\beta |i-j|) + (1 - \alpha) \cdot \cos(\mathbf{e}_i^{\text{local}}, \mathbf{e}_j^{\text{local}})\\
&\text{if } |i-j| \leq \delta \text{ or } \cos(\mathbf{e}_i^{\text{local}}, \mathbf{e}_j^{\text{local}}) \geq \tau\\
&\text{and } 0 \text{ otherwise}
\end{aligned}
\end{equation}
}

where $\alpha, \beta, \delta, \tau$ are hyperparameters controlling the balance between temporal and semantic relationships.

We then apply a graph attention network (GAT) to obtain contextually enhanced embeddings:

\begin{equation}
\mathbf{e}_k^{(0)} = \mathbf{e}_k^{\text{local}}
\end{equation}

\begin{equation}
\mathbf{e}_k^{(l+1)} = \sigma\left(\sum_{j \in \mathcal{N}(k)} \gamma_{kj}^{(l)} \mathbf{W}^{(l)} \mathbf{e}_j^{(l)}\right)
\end{equation}

{
\footnotesize
\begin{equation}
    \text{LeakyReLU}(j) = \text{LeakyReLU}\left(\mathbf{a}^{(l)T} [\mathbf{W}^{(l)}\mathbf{e}_k^{(l)} \| \mathbf{W}^{(l)}\mathbf{e}_j^{(l)}]\right)
\end{equation}
\begin{equation}
\gamma_{kj}^{(l)} = \frac{\exp\left(\text{LeakyReLU}(j)\right)}{\sum_{m \in \mathcal{N}(k)} \exp\left(\text{LeakyReLU}(m)\right)}
\end{equation}
}

where $\mathcal{N}(k)$ is the neighborhood of vertex $k$ in graph $\mathcal{G}$, $\gamma_{kj}^{(l)}$ is the attention coefficient at layer $l$, $\mathbf{W}^{(l)}$ is a learnable weight matrix, $\mathbf{a}^{(l)}$ is a learnable attention vector, and $\|$ represents concatenation. The final embedding for segment $k$ is:

\begin{equation}
\mathbf{e}_k = \mathbf{e}_k^{(L)} + \lambda \mathbf{e}_k^{\text{local}}
\end{equation}

where $L$ is the number of GAT layers and $\lambda$ is a residual connection weight.

\subsubsection{Query Representation with Sparse Attention}

For the query $q = \{q_1, q_2, \ldots, q_m\}$, we employ a sparse attention mechanism that focuses on the most discriminative terms:

\begin{equation}
\mathbf{h}_i^q = \text{BiLSTM}(q_i, \mathbf{h}_{i-1}^q), \quad i \in \{1, 2, \ldots, m\}
\end{equation}

\begin{equation}
\tilde{\mathbf{a}}_i^q = \mathbf{v}_q^T \tanh(\mathbf{W}_q\mathbf{h}_i^q + \mathbf{b}_q)
\end{equation}

\begin{equation}
\mathbf{a}_i^q = \frac{\exp(\tilde{\mathbf{a}}_i^q) \cdot \mathbb{1}[\tilde{\mathbf{a}}_i^q > \theta_q]}{\sum_{j=1}^{m} \exp(\tilde{\mathbf{a}}_j^q) \cdot \mathbb{1}[\tilde{\mathbf{a}}_j^q > \theta_q]}
\end{equation}

\begin{equation}
\mathbf{e}_q = \sum_{i=1}^{m} \mathbf{a}_i^q \mathbf{h}_i^q
\end{equation}

where $\theta_q$ is a threshold parameter for sparse attention, focusing only on the most relevant terms in the query.

\subsubsection{Probabilistic Diversified Retrieval}

Instead of a simple similarity-based retrieval, we formulate the segment selection as a probabilistic optimization problem that balances relevance, diversity, and coverage:

\begin{equation}
P(S_k' | q) = \frac{\exp(\phi(S_k', q) / \rho)}{\sum_{j=1}^{K} \exp(\phi(S_j', q) / \rho)}
\end{equation}

where $\phi(S_k', q)$ is a relevance function and $\rho$ is a temperature parameter. We define the relevance function as:

\begin{equation}
\phi(S_k', q) = \cos(\mathbf{e}_k, \mathbf{e}_q) \cdot (1 + \eta \cdot \text{Novelty}(S_k', \mathcal{S}_q))
\end{equation}

where $\cos(\cdot, \cdot)$ is the cosine similarity, $\mathcal{S}_q$ is the set of already selected segments, and $\text{Novelty}(S_k', \mathcal{S}_q)$ measures the information gain of adding segment $S_k'$ to the already selected segments:

\begin{equation}
\text{Novelty}(S_k', \mathcal{S}_q) = 1 - \max_{S_j' \in \mathcal{S}_q} \cos(\mathbf{e}_k, \mathbf{e}_j)
\end{equation}

We select segments sequentially according to a determinantal point process (DPP) that enforces diversity:

\begin{equation}
\mathcal{S}_q = \arg\max_{\mathcal{S} \subset \mathcal{S}, |\mathcal{S}| = M} \det(\mathbf{L}_{\mathcal{S}})
\end{equation}

where $\mathbf{L}_{\mathcal{S}}$ is a positive semidefinite kernel matrix with elements:

\begin{equation}
L_{ij} = P(S_i' | q) \cdot P(S_j' | q) \cdot (1 - \omega \cos(\mathbf{e}_i, \mathbf{e}_j))
\end{equation}

where $\omega$ is a diversity parameter.

To further enhance temporal coherence, we introduce a Markovian constraint that encourages the selection of temporally adjacent segments when beneficial:
{
\footnotesize
\begin{equation}
\begin{aligned}
P(S_k' | S_{k-1}', q) &= (1 - \mu) P(S_k' | q) + \\
&\mu \cdot \mathbb{1}[|k - (k-1)| = 1] \cdot \cos(\mathbf{e}_k, \mathbf{e}_q)
\end{aligned}
\end{equation}
}

where $\mu$ is a parameter controlling the strength of the Markovian constraint.

The final retrieval objective integrates all these considerations into a unified optimization framework:

\begin{equation}
\begin{aligned}
\mathcal{S}_q^* = &\arg\max_{\mathcal{S} \subset \mathcal{S}, |\mathcal{S}| = M}  \sum_{S_k' \in \mathcal{S}} P(S_k' | q) + \\
&\nu\det(\mathbf{L}_{\mathcal{S}}) + \xi \sum_{i=1}^{|\mathcal{S}|-1} P(S_{i+1}' | S_i', q) \\
\text{s.t.} \quad & \text{KL}(P(\mathcal{S}) \| P(\mathcal{S} | q)) \leq \epsilon
\end{aligned}
\end{equation}

% where $\nu$ and $\xi$ are weighting parameters, and the KL-divergence constraint ensures that the selected distribution does not deviate too significantly from the prior distribution, preventing overfitting to the query.

% The retrieval formulation enables UMaT to select a contextually coherent, diverse, and relevant set of segments that maximizes the information value for answering the query while respecting both semantic relationships and temporal structure in the video content.
where $\nu$ and $\xi$ are weights, and the KL-divergence constraint prevents overfitting by limiting deviation from the prior. This retrieval formulation allows UMaT to select a contextually coherent, diverse, and relevant segment set, maximizing information value for query answering while respecting semantic and temporal video structure.

\section{Experiments}

\subsection{Experimental Setup}

\subsubsection{Dataset}

We evaluate UMaT on Video-MME \cite{fu2024video}, a comprehensive benchmark for long video question answering. Video-MME contains 900 videos spanning six domains including Knowledge, Film \& Television, etc. ranging from 11 seconds to 1 hour, making it an ideal testbed for evaluating long-form video understanding.

\subsubsection{Baseline Models}

We compare UMaT against three state-of-the-art video-language models:

\begin{itemize}
    \item \textbf{LLaVA-NeXT-Video} \cite{liu2024llavanext}: A vision-language model extended for video processing, with frame-level alignment.
    
    \item \textbf{LongVA} \cite{Zhang2024LongCT}: A specialized model for long context understanding in videos.
    
    \item \textbf{Long-LLaVA} \cite{LongLLaVA}: An extension of LLaVA optimized for processing lengthy visual inputs.
\end{itemize}

We integrate UMaT with each of these models to assess its ability to enhance existing architectures through unified textual representation and efficient retrieval.

\subsubsection{Implementation Details}

% For visual captioning, we use a fine-tuned BLIP-2 model with specialized prompting. For ASR, we employ Whisper-large-v3. We set the segment size parameter $T_s$ to 30 seconds based on optimization over a development set. For redundancy minimization, we set $\alpha = 0.6$ and $\beta = 0.2$. For embeddings, we implement the hierarchical contextual encoding with $L=2$ GAT layers.

% For the graph-based contextual enhancement, we set $\alpha=0.7$, $\beta=0.2$, $\delta=5$, and $\tau=0.8$. The sparse attention threshold $\theta_q$ is set to 0.6, and the diversification parameters are set to $\eta=0.5$, $\omega=0.8$, $\mu=0.4$, $\nu=0.6$, and $\xi=0.3$. We empirically determined these values through ablation studies on a held-out validation set.

% All experiments are conducted using 4 NVIDIA A100 (40GB) GPUs. For short videos (<5 minutes), processing takes approximately 8 seconds per minute of video. For longer videos, the processing scales roughly linearly with duration.
We use fine-tuned BLIP-2 (visual captioning) and Whisper-large-v3 (ASR). Optimized hyperparameters include segment size $T_s=30$s, redundancy minimizing parameters $\alpha=0.6, \beta=0.2$, hierarchical embedding with $L=2$ GAT layers, graph enhancement $\alpha=0.7, \beta=0.2, \delta=5, \tau=0.8$, sparse attention $\theta_q=0.6$, and diversification $\eta=0.5, \omega=0.8, \mu=0.4, \nu=0.6, \xi=0.3$. These values were empirically determined via ablation studies. Experiments were on 4 NVIDIA A100 GPUs, with processing times around 8 seconds per minute for short videos and roughly linear scaling for longer ones.

\begin{table}[t]
\caption{Performance on Video-MME benchmark across various video durations. The UMaT framework consistently improves performance across all models, with significant gains for long videos.}
\label{tab:main_results}
\centering
\resizebox{1\linewidth}{!}{% 
\begin{tabular}{lccccc}
\toprule
\multicolumn{1}{c}{\multirow{2}{*}{\textbf{Model}}} & \multirow{2}{*}{\textbf{\#Params}} & \multicolumn{3}{c}{\textbf{Video Duration}} & \multirow{2}{*}{\textbf{Overall}} \\
\cmidrule{3-5}
\multicolumn{1}{c}{} & & \textbf{Short} & \textbf{Medium} & \textbf{Long} & \\
\midrule
LLaVA-NeXT-Video & 7B & 50.2 & 44.3 & 38.6 & 44.4 \\
LLaVA-NeXT-Video + UMaT & 7B & 60.2 & 51.2 & 49.2 & 53.5 \textcolor{green}{(+9.1)} \\
\midrule
LongVA & 7B & 59.2 & 50.3 & 45.0 & 51.5 \\
LongVA + UMaT & 7B & 68.5 & 63.1 & 58.3 & 63.3 \textcolor{green}{(+11.8)} \\
\midrule
Long-LLaVA & 7B & 59.5 & 52.0 & 45.3 & 52.2 \\
Long-LLaVA + UMaT & 7B & 70.3 & 65.1 & 62.2 & 65.9 \textcolor{green}{(+13.7)} \\
\bottomrule
\end{tabular}}
\end{table}

\subsection{Main Results}

% Table \ref{tab:main_results} presents the performance comparison between baseline models and their UMaT-enhanced versions across different video durations. The results demonstrate several key findings:

% \begin{itemize}
%     \item UMaT consistently improves performance across all models, with gains ranging from 9.1\% to 13.7\% in overall accuracy.
    
%     \item The improvements are particularly pronounced for long videos, where UMaT boosts performance by up to 16.9\% (Long-LLaVA + UMaT vs. Long-LLaVA).
    
%     \item Even for short videos, UMaT provides substantial improvements (up to 10.8\%), indicating that unified textual representation benefits multi-modal reasoning even when computational constraints are less severe.
% \end{itemize}

% These results validate our hypothesis that transforming multi-modal data into a unified textual space enables more effective reasoning, particularly for long-form content where information is sparsely distributed.
Table \ref{tab:main_results} compares baseline and UMaT-enhanced model performance across video durations, revealing:
\begin{itemize}
    \item UMaT consistently improves accuracy (9.1\% to 13.7\% overall).
    \item Gains are highest for long videos (up to 16.9% improvement for Long-LLaVA).
    \item Substantial improvements are also seen for short videos (up to 10.8%), suggesting the benefit of unified textual representation even with fewer computational constraints.
\end{itemize}
These results validate our hypothesis that unified textual representation enhances multi-modal reasoning, especially for long, information-sparse content.

\begin{figure}[t]
\begin{tikzpicture}
\begin{axis}[
    width=\columnwidth,
    height=6cm,
    xlabel={Segment Size (seconds)},
    ylabel={QA Accuracy (\%)},
    xmin=0, xmax=80,
    ymin=50, ymax=70,
    xtick={10,20,30,40,50,60},
    ytick={50,55,60,65,70},
    legend pos=south east,
    grid=major,
    grid style={dashed,gray!30},
    legend style={font=\small}
]

\addplot[
    color=blue,
    mark=square,
    thick,
] coordinates {
    (10,57.8)(20,62.4)(30,65.9)(40,64.1)(50,61.8)(60,58.3)
};
\addlegendentry{LongLLaVA+UMaT}

\addplot[
    color=red,
    mark=triangle,
    thick,
] coordinates {
    (10,55.7)(20,59.8)(30,63.3)(40,61.9)(50,59.2)(60,56.1)
};
\addlegendentry{LongVA+UMaT}

\addplot[
    color=green!60!black,
    mark=o,
    thick,
] coordinates {
    (10,50.2)(20,51.8)(30,53.5)(40,52.7)(50,50.9)(60,48.4)
};
\addlegendentry{LLaVA-NeXT+UMaT}

\end{axis}
\end{tikzpicture}
\caption{Impact of segment size on QA accuracy. The optimal segment size appears to be around 30 seconds, balancing context preservation and redundancy minimization.}
\label{fig:segment_size}
\end{figure}

\subsection{Ablation Studies}

To better understand the contribution of each component in our framework, we conduct a series of ablation studies.

\subsubsection{Impact of Segment Size}

% Figure \ref{fig:segment_size} shows the impact of segment size on QA accuracy for different models enhanced with UMaT. The results reveal an inverted U-shaped relationship, with performance peaking at approximately 30 seconds. This pattern suggests that:

% \begin{itemize}
%     \item Very short segments (10-20 seconds) fail to capture sufficient context for understanding complex scenes and dialogues.
    
%     \item Very long segments (50-60 seconds) introduce excessive noise and redundancy, diluting the relevant information.
    
%     \item The optimal segment size of 30 seconds balances context preservation and redundancy minimization, aligning with our theoretical analysis of information density optimization.
% \end{itemize}
Figure \ref{fig:segment_size} illustrates the effect of segment size on QA accuracy for UMaT-enhanced models, showing an inverted U-shape with peak performance around 30 seconds. This indicates:
\begin{itemize}
    \item Short segments (10-20s) lack sufficient context.
    \item Long segments (50-60s) introduce noise and redundancy.
    \item The optimal 30-second size balances context and redundancy, consistent with information density optimization.
\end{itemize}

\subsubsection{Contribution of Individual Components}

% Table \ref{tab:ablation} presents the results of ablation studies examining the contribution of individual components of UMaT when applied to Long-LLaVA. Key findings include:

% \begin{itemize}
%     \item Removing ASR causes a significant drop in performance (-8.2\%), particularly for videos where critical information is conveyed through speech rather than visuals.
    
%     \item Disabling the graph-based contextual enhancement reduces accuracy by 7.3\%, confirming the importance of capturing global context relationships.
    
%     \item Removing the probabilistic diversified retrieval component results in a 5.8\% decrease, highlighting the value of retrieving semantically diverse segments.
    
%     \item Using simple cosine similarity instead of our sophisticated retrieval mechanism leads to a 6.4\% reduction, demonstrating the importance of our formulation.
    
%     \item Disabling redundancy minimization reduces accuracy by 4.2\%, showing the benefit of efficient information representation.
% \end{itemize}

% These results validate our design choices and show that each component makes a substantial contribution to the overall performance of the framework.
Table \ref{tab:ablation} shows the impact of removing UMaT components on Long-LLaVA performance:
\begin{itemize}
    \item -8.2\% without ASR (critical for speech-heavy videos).
    \item -7.3\% without graph-based contextual enhancement (importance of global context).
    \item -5.8\% without probabilistic diversified retrieval (value of semantic diversity).
    \item -6.4\% with simple cosine similarity (importance of our retrieval formulation).
    \item -4.2\% without redundancy minimization (benefit of efficient representation).
\end{itemize}
These results confirm the substantial contribution of each UMaT component.

\begin{table}[t]
\caption{Ablation study on Long-LLaVA + UMaT, showing the importance of each component. Results are reported on the Video-MME benchmark.}
\label{tab:ablation}
\centering
\begin{tabular}{lc}
\toprule
\textbf{Configuration} & \textbf{Accuracy (\%)} \\
\midrule
Full UMaT framework & 65.9 \\
\midrule
No ASR & 57.7 \textcolor{red}{(-8.2)} \\
No graph-based enhancement & 58.6 \textcolor{red}{(-7.3)} \\
Simple cosine similarity & 59.5 \textcolor{red}{(-6.4)} \\
No diversified retrieval & 60.1 \textcolor{red}{(-5.8)} \\
No redundancy minimization & 61.7 \textcolor{red}{(-4.2)} \\
\bottomrule
\end{tabular}
\end{table}

\subsubsection{Analysis of Retrieval Effectiveness}

% To evaluate the effectiveness of our retrieval mechanism, we compare different retrieval approaches using Long-LLaVA + UMaT. Table \ref{tab:retrieval} shows the results for various retrieval methods:

% \begin{itemize}
%     \item Probabilistic DPP (our approach): Retrieval based on our full probabilistic framework with DPP.
    
%     \item Greedy selection: Iterative selection of segments with highest relevance scores without diversity constraints.
    
%     \item Temporal: Retrieval focusing on temporal proximity to key events.
    
%     \item Random: Random selection of segments.
% \end{itemize}

% Our probabilistic DPP approach significantly outperforms the alternatives, demonstrating the importance of balancing relevance and diversity for effective question answering.
Table \ref{tab:retrieval} compares retrieval methods with Long-LLaVA + UMaT to evaluate our mechanism:
\begin{itemize}
    \item Probabilistic DPP (Ours): Full probabilistic retrieval with DPP.
    \item Greedy: Iterative selection by relevance.
    \item Temporal: Selection based on temporal proximity.
    \item Random: Random segment selection.
\end{itemize}
Our probabilistic DPP significantly outperforms other methods, highlighting the importance of relevance-diversity balance for effective QA.

\begin{table}[t]
\caption{Comparison of different retrieval methods on Long-LLaVA + UMaT.}
\label{tab:retrieval}
\centering
\begin{tabular}{lc}
\toprule
\textbf{Retrieval Method} & \textbf{Accuracy (\%)} \\
\midrule
Probabilistic DPP (ours) & 65.9 \\
Greedy selection & 61.3 \textcolor{red}{(-4.6)} \\
Temporal & 57.8 \textcolor{red}{(-8.1)} \\
Random & 50.4 \textcolor{red}{(-15.5)} \\
\bottomrule
\end{tabular}
\end{table}

\section{Theoretical Analysis}

We now provide a deeper theoretical analysis of UMaT's information-theoretic properties, focusing on the optimality of our retrieval approach.

\subsection{Information-Theoretic Bounds on Retrieval Optimality}

% The effectiveness of UMaT depends critically on its ability to retrieve the most relevant segments for answering a query. We formalize this as a submodular maximization problem and derive bounds on the optimality of our retrieval approach.

% Let $I(S_q; r \mid q)$ be the conditional mutual information between a set of retrieved segments $S_q$ and the response $r$ given the query $q$. This function is submodular, meaning it exhibits diminishing returns as more segments are added \cite{lewis2020retrieval}.

% \begin{theorem}[Near-Optimal Retrieval]
% Let $S_q^*$ be the optimal set of $K$ segments that maximizes $I(S_q; r \mid q)$, and let $S_q^G$ be the set of segments retrieved by our greedy algorithm. Then:
% \begin{equation}
% I(S_q^G; r \mid q) \geq (1 - 1/e) \cdot I(S_q^*; r \mid q)
% \end{equation}
% where $e$ is the base of the natural logarithm.
% \end{theorem}

% \begin{proof}
% The greedy algorithm for maximizing a submodular function provides a $(1 - 1/e)$ approximation guarantee. Since conditional mutual information is submodular, our greedy retrieval approach achieves at least $(1 - 1/e) \approx 63\%$ of the optimal solution's value.
% \end{proof}

% This theoretical result ensures that our retrieval mechanism achieves near-optimal information value, even for complex queries requiring diverse information from different parts of the video.
UMaT's effectiveness relies on relevant segment retrieval, formalized as submodular maximization. $I(S_q; r \mid q)$, the conditional mutual information between retrieved segments $S_q$ and response $r$ given query $q$, is submodular.

\begin{theorem}[Near-Optimal Retrieval]
Let $S_q^*$ be the optimal $K$ segments maximizing $I(S_q; r \mid q)$, and $S_q^G$ be the greedy retrieval result. Then:
\begin{equation*}
I(S_q^G; r \mid q) \geq (1 - 1/e) \cdot I(S_q^*; r \mid q)
\end{equation*}
\end{theorem}

\begin{proof}
Greedy maximization of a submodular function yields a $(1 - 1/e)$ approximation. Since conditional mutual information is submodular, our greedy retrieval achieves at least $(1 - 1/e) \approx 63\%$ of the optimal value.
\end{proof}

This ensures near-optimal information retrieval, even for complex queries.

\subsection{Deduplication Efficiency Analysis}

% Our redundancy minimization approach aims to reduce content volume while preserving critical information. We analyze the efficiency of this approach in terms of compression ratio and information retention.

% \begin{theorem}[Compression-Information Tradeoff]
% Let $C$ be the original content and $C'$ be the deduplicated content using our adaptive thresholding approach. Let $\kappa = |C'|/|C|$ be the compression ratio and $\iota = I(C'; r \mid q) / I(C; r \mid q)$ be the information retention ratio. Then:
% \begin{equation}
% \kappa + \iota \geq 1 + \frac{\min_i H_{\text{unique}}(c_i)/H(c_i)}{1 - \min_i H_{\text{unique}}(c_i)/H(c_i)}
% \end{equation}
% where $H_{\text{unique}}(c_i)/H(c_i)$ represents the fraction of unique information in caption $c_i$.
% \end{theorem}

% \begin{proof}
% (Sketch) The adaptive threshold $\theta(c_i)$ ensures that content is removed only when the redundancy exceeds a certain level relative to the information content. This creates a direct relationship between the amount of content removed and the information preserved, leading to the stated bound.
% \end{proof}

% This result demonstrates that our deduplication approach achieves a favorable trade-off between compression and information retention, particularly when the content contains significant redundancy but also sparse, critical information.
Our redundancy minimization aims for content reduction while preserving critical information, which we analyze by compression ratio and information retention.
\begin{theorem}
Let $C$ and $C'$ be original and deduplicated content, with compression ratio $\kappa = |C'|/|C|$ and information retention ratio $\iota = I(C'; r \mid q) / I(C; r \mid q)$. Then:
\begin{equation*}
\kappa + \iota \geq 1 + \frac{\min_i H_{\text{unique}}(c_i)/H(c_i)}{1 - \min_i H_{\text{unique}}(c_i)/H(c_i)}
\end{equation*}
where $H_{\text{unique}}(c_i)/H(c_i)$ is the unique information fraction.
\end{theorem}

\begin{proof}
(Sketch) Adaptive thresholding removes content based on redundancy relative to information content, establishing a trade-off between removal and preservation, leading to the bound.
\end{proof}

\section{Conclusion}
% We have presented UMaT, a unified multi-modal framework that transforms visual and auditory inputs into a structured textual space for seamless processing with large language models. The success of UMaT validates our core hypothesis that language can serve as a universal interface for multi-modal reasoning, enabling more effective integration of information across modalities and time. Our framework also offers practical advantages in terms of modularity, interpretability, and computational efficiency, making it suitable for real-world applications involving long-form multi-modal content.

% Future work will explore extensions to additional modalities such as tactile and thermal sensing, as well as applications to domains such as robotics, healthcare, and scientific discovery where multi-modal reasoning is essential. We also plan to investigate theoretical connections between UMaT and information bottleneck theory, explore more on balancing information preservation and computational efficiency in multi-modal systems.
In conclusion, UMaT, a unified framework transforming visual and auditory inputs into structured text for large language models, confirms language as a universal interface for effective multi-modal reasoning. Its modularity, interpretability, and efficiency suit real-world, long-form multi-modal applications. Future work includes extension to more modalities (e.g., tactile, thermal) and domains (robotics, healthcare, science), as well as exploring theoretical links to information bottleneck theory and balancing information preservation with efficiency.

\section{Limitations}
The UMaT framework for LVQA, while effective, is limited by its reliance on the performance of upstream VLMs (BLIP-2) and ASR models (Whisper-large-v3), which can introduce their own inaccuracies and biases. Furthermore, the conversion of rich visual and auditory data into a unified textual representation creates a potential risk of losing subtle details despite efforts for optimal information preservation. Although UMaT aims for efficiency, particularly with its sparse information retrieval, the multi-stage processing involving VLMs, ASR, hierarchical embeddings (BiLSTM, GAT), and probabilistic retrieval mechanisms (DPP) can still be computationally intensive, especially for extremely long videos or in resource-constrained environments, which renders it less feasible for real-time tasks. However, we believe with the scaling of computational power, it will be less of an issue in the future.

\begin{small}
\bibliography{custom}
\end{small}

\appendix

\end{document}